\title{The complexity of unsupervised learning of lexicographic preferences}
\documentclass{article}\pdfpagewidth=8.5in\pdfpageheight=11in
\usepackage{ijcai22,times}

\author{
Hélène Fargier$^1$
\and
Pierre-Fran\c cois Gimenez$^2$\and
Jérôme Mengin$^3$\And
Ngoc Bao Nguyen$^4$
\affiliations
$^{1,3}$IRIT, Université de Toulouse, CNRS, Toulouse INP, UT3, Toulouse, France\\
$^2$CentraleSupélec, Univ. Rennes, IRISA\\
$^4$INSA Toulouse\\
\emails
\{helene.fargier,jerome.mengin\}@irit.fr,
pierre-francois.gimenez@centralesupelec.fr,
nbnguyen@etud.insa-toulouse.fr
}

\usepackage{tikz}\usetikzlibrary{arrows,shapes,positioning,topaths}

\usepackage{ifoption,maclasse-unicodes,maclasse-maths}
\usepackage{natbib}
\usepackage{amsmath,stmaryrd,mathtools}

\usepackage[appendix=strip,bibliography=common]{apxproof}
\newtheorem{theorem}{Theorem}
\newtheoremrep{proposition}[theorem]{Proposition} 	
 \newtheorem{corollary}[theorem]{Corollary}  \theoremstyle{definition} \newtheorem{definition}{Definition} \newtheorem{example}{Example}

\newcommand\compactmath{\thinmuskip0mu\medmuskip0mu\thickmuskip0mu\arraycolsep1pt}

\usepackage{enumitem,algorithm}
\newlist{steps}{enumerate}4
\setlist[steps,1]{label=\arabic*.,nosep,leftmargin=2em}
\setlist[steps,2]{label=\alph{*}.,nosep,leftmargin=*}
\setlist[steps,3]{label=\roman{*}.,nosep,leftmargin=*}
\setlist[steps,4]{label=\Roman*.,nosep,leftmargin=*}
\newcommand\sousitem{\par\hangindent 40\p@ \hspace*{20\p@}}
\newcommand\noitem{\item[]\hskip-\leftmargin}
\usepackage{comment}

\newif\iftc \tcfalse%
\makeatletter\if@twocolumn\tctrue\fi\makeatother

\newcommand\pfg[1]{}

\DeclareMathSymbol{:}{\mathop}{operators}{"3A}

\let\dom\underline
\newcommand\domX{\dom{\cal X}}

\newcommand\erank{\overline{§rank§}}

\newcommand{\probS}{p_{\cal S }}

\DeclareUnicodeCharacter{0306}{\mathaccent"015\relax}
\DeclareUnicodeCharacter{227B}{\succ}

\begin{document}

\maketitle

\begin{abstract}
This paper considers the task of learning users' preferences on a combinatorial set of alternatives, as generally used by online configurators, for example. In many settings, only a set of selected alternatives during past interactions is available to the learner.
\cite{FargierGimenezMengin:aaai18} propose an approach to learn, in such a setting, a model of the users' preferences that ranks previously chosen alternatives as high as possible; and an algorithm to learn, in this setting, a particular model of preferences: lexicographic preferences trees (LP-trees).
In this paper, we study complexity-theoretical problems related to this approach. We give an upper bound on the sample complexity of learning an LP-tree, which is logarithmic in the number of attributes. We also prove that computing the LP tree that minimises the empirical risk can be done in polynomial time when restricted to the class of linear LP-trees.
\end{abstract}

\section{Introduction}

Modern, interactive decision support systems like recommender
systems or configurators often handle a very
large set of possible decisions/alternatives. The task
of finding the  alternatives that best suit their
preferences can be challenging for users, but the
system can guide them towards their optimal decision
if it has some knowledge of their preferences. In many
settings, the users'
preferences   are not  known  in
 advance. This is especially the case with systems that enable
 anonymous users to browse the catalogues: such systems must
be able to acquire users' preferences.

Preference learning has emerged as an important field; many interesting results are reported in, e.g., the book edited by~\cite{FurnkranzHullermeier:book11}, or the proceedings of recent Preference Learning or DA2PL (Decision Aid to Preference Learning) workshops. A general problem is: given a set of observed preferences, induce a model of preferences that best explains these observations within a certain class of models. As input, it is often assumed that the observed preferences are given as a set of pairwise comparisons   or partial rankings of alternatives  \citep{Joachims:kdd02};
or can be elicitated online by asking the user to choose between two alternatives \citep{ViappianiAlJAIR06,KoricheZ09}.

But in some circumstances, such input is not available. This is especially the case with some anonymous on-line configurators, where little information is stored about interactions.
However, e-commerce companies generally keep a history of past sales. Users have chosen sold items, so they are probably ranked high in their preferences, but not necessarily at the very top. Indeed, a user may eventually choose an item that is not the optimal one in her preference order. For instance, because of the difficulty in grasping all possible options, a phenomenon called ``mass confusion'' \citep{huffman1998variety}, because of the influence of an advertisement, or because her preferred item is unavailable. Yet, this list of highly ranked items does provide information about the users' preferences.

\cite{FargierGimenezMengin:aaai18} propose a model for learning preferences in such settings: the learning algorithm receives a multiset of alternatives that past users have chosen and induces a ranking of the alternatives to guide future users in their exploration of the set of possible alternatives.
If, for instance, the colour red appears more often in the sales history than the colour yellow, then we want to induce a model that ranks alternatives with the colour red higher than alternatives with the colour yellow -- maybe in association with some other criteria. This is similar to the usual, unsupervised setting when learning Bayesian networks~\citep[see e.g.][]{Neapolitan:book03}. This unsupervised setting is convenient from a machine learning point of view, as data is usually easy to obtain and does not necessitate any input labeling. The alternatives that users in past sessions have chosen are called ``positive examples'' in \citep{FargierGimenezMengin:aaai18}, as opposed to possible ``negative examples'' that one could have in a setting where we would also have information about alternatives \emph{rejected} by past users.

Research on the representation and learning of preferences has brought forward several types of models. Numerical models, like linear ranking functions or additive utilities \citep{Joachims:kdd02,Freundetal:jmlr03,ScFaVe1995.1,GonzalesPerny04,Braziunas05localutility}, are rich families of models, especially if one allows high-dimensional feature spaces.
Research in Artificial Intelligence has also brought forward ordinal models, like
CP-nets~\citep{Boutilier04cp-nets} and several extensions or variants. 
 Lexicographic preferences are another family of ordinal models. This kind of preference is based on the importance of the attributes: when comparing two outcomes, their values for the most important attribute are compared; if the two outcomes have different values for that attribute, then the one with the preferred value is deemed preferable to the other; otherwise  one looks at the next most important attribute, and so on.
This model can be extended by allowing the preferences on the values of an issue  to depend on the values of more important ones. The relative importance of issues is no longer a linear order, but a  ``lexicographic preference'' tree \citep{fraser:ieeeconf93,fraser:theo-dec94,Wilson:ecai06,WallaceW:ann-OR09,Boothetal:ecai10}. 

Lexicographic preference trees have several advantages over other preference representation models. First, this is an ordinal model, which is sufficient to represent a ranking of alternatives. Furthermore,
they are generally an accurate representation of human behaviours \citep{GigerenzerG:psych-reviewG96}. Finally, one can quickly (in polytime) perform some interesting requests for recommendations, such as finding an optimal object or an optimal value for some attribute \citep{FargierMengin:aamas21}.

Learning lexicographic preference models with pairwise comparisons as inputs has been studied by e.g.~\cite{SchmittMartignon:jmlr06,Dombietal:ejor07,Yamanetal:icml08}, while \cite{Boothetal:ecai10,BrauningHullermeier:pl12,Brauningetal:ejor17,LiuTruszczynski:aaai15} studied learning of  lexicographic preference trees. More recently, \cite{FargierGimenezMengin:aaai18} have proposed a greedy algorithm for learning lexicographic preference trees from sales history. They reported some promising results on experimentation on both synthetic data and an industrial dataset from a car manufacturer (Renault), using clustering in a pre-processing step and a pruning pass in a post-processing step.

This paper proposes a complexity theoretical analysis of the approach proposed by \cite{FargierGimenezMengin:aaai18}. We derive an upper bound on the sample complexity of computing the optimal LP-tree w.r.t. a given sample of chosen alternatives. This complexity is, in particular, logarithmic in the number of attributes. We also prove that computing this optimal LP-tree can be done in polynomial time when restricted to the class of ``linear'' LP-trees (which correspond to usual lexicographic preferences). Finally, we propose an algorithm for computing it for the more expressive class of LP lists \citep{Brauningetal:ejor17}, where several attributes can be at the same importance level.



The paper is structured as follows. The next section recall some background on combinatorial domains, LP-trees and the learning model introduced by \cite{FargierGimenezMengin:aaai18}. In section~\ref{sect:rank} we derive some results on the computation of the rank of alternatives w.r.t some LP-tree. The following three sections are devoted to three classes of LP-trees, in order of increasing generality.

\section{Background and notations}

\subsection{Combinatorial Domain}

We consider a combinatorial domain over a finite set $ \cal X $ discrete attributes that characterise the possible alternatives, each attribute $ X ∈ \cal X  $ having a finite set of possible values $ \dom X $; we assume that $ | \dom X | ≥ 2 $ for every $ X ∈ \cal X $; then $ \dom{\cal X } $ denotes the Cartesian product of the domains of the attributes in $ \cal X  $, its elements are called alternatives; we often use the symbols $ o $, $ o' $, $ o₁ $, $ o₂ $, … to denote alternatives.
In the sequel, $ n $ is the number of attributes in $ \cal X $, and $ d $ is a bound on the size of the domains of the attributes: for every $ X ∈ \cal X  $, $ \card{\dom X} ≤ d $.

For a subset $ U $ of $ \cal X  $, we will denote by $ \dom U $ the cartesian product of the domains of the attributes in $ U $, every $ u ∈ \dom U $ is an instantiation of $ U $, or partial instantiation (of $ \cal X  $). If $ v $ is an instantiation of some $ V ⊆ \cal X  $, $ v[U] $ denotes the restriction of $ v $ to the attributes in $ V \cap U $; we say that instantiation $ u ∈ \dom U $ and $ v $ are compatible if $ v[U∩V] = u[U∩V] $; if $ U ⊆ V $ and $ v[U] = u $, we say that $ v $ extends $ u $.

Given a partial instantiation $ u $, $ §Var§(u) $ denotes the set of attributes, the values of which appear in $ u $.

\subsection{Preference relation}


In this paper, a preference relation is a linear order over  $ \domX $, that is, a total, transitive, irreflexive binary relation over $ \domX $, often denoted with curly symbol $ ≻ $. For alternatives $ o,o' ∈ \domX $, $ o ≻ o' $ indicates that $ o $ is strictly more preferred to $ o' $. 

Because we consider linear orders over $ \domX $, we can define the \emph{rank} of $o ∈ \domX $ w.r.t. $ ≻ $: $ §rank§(≻,o) = 1 + $ the number of outcomes strictly preferred to $ o $, so that the most preferred outcome has rank 1, the least preferred has rank $ \card{\domX} $:
$$
  §rank§(\succ,o) = 1+\card{\{o' \in \domX \mid o' \succ o\}}.
$$

\subsection{Learning model}

We consider an  unknown target linear order $ ̆≻ $ over $ \domX $, representing the preferences of a decision maker, or of a group of decision makers: whenever these decision makers have to make a decision, they choose an alternative according to $  ̆≻ $; these decisions are not always the ``top'' alternative, maybe because of the difficulty of finding it, or because of some context making it not available for instance. However, we consider that there is some probability distribution $p$ of drawing alternatives of $ \domX $, unknown but supposed to be decreasing w.r.t. to $ §rank§( ̆≻,⋅) $:
if $o > o’$, then $p(o) ≥ p(o’)$ and $rank(o) < rank(o’)$.


We want to learn a (representation of) a linear order $ ≻ $ that is as close as possible to $  ̆≻ $, that can be used to give good answers to queries about $  ̆≻ $ (for instance: “what is the optimal outcome?”, or “is $ o $ preferred to $ o' $?”)

In order to have a relevant measure of how close $ ≻ $ is to $ ̆≻$ , \cite{FargierGimenezMengin:aaai18} introduce the notion of \emph{ranking loss}, defined as the normalised difference between the expected ranks of the two relations according to the ground probability $ p  $:
\begin{align}
  & §rloss§_p(≻, ̆≻)\nonumber
  \\&\qquad  \label{eq:rloss=exp-rank-diff}
    = \frac 1 {\card{\domX}} \left( E_p[§rank§( ≻,⋅)] - E_p[§rank§( ̆≻,⋅)] \right)
\\  &\qquad \label{eq:rloss=weighted-spearman}
    = \frac 1 {\card{\domX}}  \displaystyle  \sum_{o ∈ \domX }p(o) \big(§rank§(≻, o)- §rank§( ̆≻, o) \big)
\end{align}

The aim of a learning process in this setting is to find, given an unknown target $ ̆≻ $ and an associated, unknown, probability distribution $ p $, a linear order $ ≻ $ that minimises this ranking loss.
Eq.~\ref{eq:rloss=exp-rank-diff} indicates that this is equivalent to finding a linear order that has minimal expected rank.

Eq.~\ref{eq:rloss=weighted-spearman} shows that this loss bears some similarity with the Spearman distance between linear orders; here however, the contribution of the rank difference for each alternative $ o $ is weighted with its probability of being drawn. This is because we want to learn a model that orders more accurately preferred alternatives; it is less significant to make mistakes with the ordering of less preferred alternatives. Also, there will not be much information about the alternatives in the tail of the distribution, which are the less preferred one.

\begin{proposition}[\citealt{FargierGimenezMengin:aaai18}]\label{rloss-pos}
Let $ ̆≻ $ and $ ≻ $ be two linear orders over $ \domX $ and $p $ a probability distribution decreasing w.r.t. $ §rank§( ̆≻, ⋅) $. Then $ 0  ≤ §rloss§_p(≻,  ̆≻) < 1 $. Furthermore, if $p $ is \emph{strictly} decreasing w.r.t. $§rank§( ̆≻, ⋅)$, then $ §rloss§_p(≻,  ̆≻) = 0 $ iff $̆ ≻ = ≻  $.
\end{proposition}

The target preference is unknown, and we try to learn it from some data. In our context, a sample $ \cal S  $ is a multiset of alternatives, that may be, in the case of e-commerce for instance, a list of items that have been chosen by users of the system. Thus, $ \cal S  $ is considered to be representative of the preference expressed with the target linear order: every alternative can have several occurrences in $ \cal S  $, and the higher the rank of an alternative $ o $ in $  ̆≻ $ is, the more likely it is to find $ o $ in $ \cal S  $, and the higher $ p(o) $ is, where $ p $ is the unknown, ground probability, supposed to be decreasing w.r.t. to $ §rank§( ̆≻,⋅)$. In other words, alternatives in $ \cal S $ are supposed to be drawn with replacement from $ \cal X $ according to $p$.

Since the target preference is unknown, we cannot measure the ranking loss of an induced preference. However, since minimising the ranking loss of the induced preference amounts to minimising the sum of the ranks of the outcomes weighted by their probabilities of being drawn, we aim to minimise it by minimising the empirical mean rank of the training sample $ \cal S  $.
For alternative $ o ∈ \domX $, $ m(S,o) $ denotes the multiplicity of alternative $ o $ in $ \cal S $, that is, the number of occurrences of $ o $ in $ \cal S $. We can define $ \probS $ to be the  empirical probability distribution over $ \domX $ such that $  \probS(o) = m(\cal S,o)/\card{\cal S } $ (in particular, $  \probS(o) = 0 $ if $ o ∉ \cal S  $). Then the empirical mean rank is defined as follows:
$$
	\erank(≻,\cal S ) = 
	   \sum_{o \in \cal S } \probS(o) × §rank§(≻,o)
	   = E_{\probS}[§rank§(≻,⋅)] 
$$


In the sequel, given $ \cal S  $, $ ≻^∗ $ denotes a linear order that minimises, often within a given class $ \cal C  $ of linear orders that will be clear from the context, the empirical mean rank : $  ≻^∗ ∈ \cal C  $ and for every $ ≻ ∈ \cal C  $, $ \erank(≻^*,\cal S ) ≤ \erank(≻,\cal S ) $. Or, equivalently:
$$
  ≻^∗ = §argmin§_{≻ ∈ \cal C }	\erank(≻,\cal S )
$$

From a complexity point of view, two questions arise :
\begin{enumerate}
\item How many examples in $ \cal S $ guarantee that $ ≻^∗ $ is probably a good approximation of $ ̆≻ $ ?
\item What is the time complexity of computing a representation of $ ≻^∗ $ ?
\end{enumerate}

The first question can be made more specific in the PAC setting: given a target linear order $  ̆≻ $; a sample of outcomes $ \cal S  $ drawn from $ \domX $ according to some probability distribution $ p $ supposed to be decreasing with $ §rank§( ̆≻,⋅) $; a class $ \cal C  $ of linear orders; and two real numbers $ 0 < δ,  ε < 1 $, what is the minimal function $ S(\cal C ,n,δ,ε) $ such that
$$
  \text{if } |\cal S | > S(\cal C ,n,δ,ε) \text{ then } Pr(§rloss§_p(≻^∗, ̆≻) ≤ ε) ≥ 1 -δ
$$
The function $ S(\cal C ,⋅,⋅,⋅) $ is called the \emph{sample complexity} of learning a linear order in the class $ \cal C  $ from ``positive'' examples (as opposed to learning from a set of pairwise comparisons). Parameter ε is the approximation that is wanted, here a bound on the ranking loss, and δ specifies the probability with which with we want to attain this approximation.

\subsection{Lexicographic preference trees}\label{sect:exple-compact-LP-tree}


In this paper, we study the complexity of learning a specific class of preference relations: the preference relations that can be represented with lexicographic preference trees, or LP-trees. LP-trees generalise lexicographic orders, which have been widely studied in decision making -- see e.g.~\cite{Fishburn:managsc74}. As an inference mechanism, they are equivalent to search trees used by \cite{Boutilieretal:compint04}, and formalised by \cite{Wilson:ecai04,Wilson:aij11}. As a preference representation, and elicitation, language, slightly different definitions for LP-trees have been proposed by \cite{Boothetal:ecai10,BrauningHullermeier:pl12,FargierGimenezMengin:aaai18}.

LP-trees provide 
a nice graphical representation of the corresponding preference relation. We illustrate that on an example before giving the formal definition that we use in this paper.

\begin{example}({Example A in \cite{Wilson:aij11},} slightly extended)\label{exple:CP-th}
I am planning a holiday, with three choices / attributes: wait til next month ($W=w$) or leave now ($W=ˉw$), going to city 1, 2 or 3 ($ C=c₁ $, $ C=c₂ $ or $ C=c₃ $), travelling by plane ($ P=p $) or by car ($P=ˉp$). The picture below 
shows an LP-tree $ φ₀ $ over $ \cal X  = \{W,C,P\} $ which defines a linear order $ ≻ $ over $ \domX $ as follows.
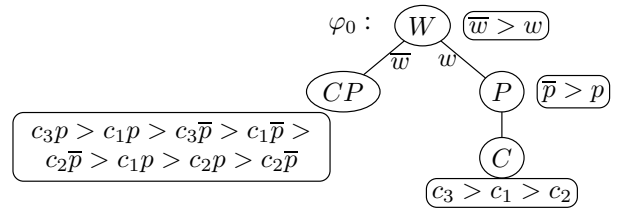
\tikzset{LPTnode/.style={draw,ellipse,inner sep=2pt,anchor=base,outer sep=0pt}}
\tikzset{LPTtable/.style={draw,rectangle,rounded corners,inner sep=2pt,anchor=base,outer xsep=5pt}}
\begin{center}
\begin{tikzpicture}[x=3em,y=-2.5em]
\node[LPTnode] (W) at (0,0) {$W$} ; \node[LPTtable,anchor=west] at (W.east) {$ˉw>w$}; \node[anchor = east] at (W.west) {$φ₀$ : \ \ } ;
\node[LPTnode] (CP) at (-1,1) {$CP$} ; \node[LPTtable,anchor=north east] at (CP.south)
  {$\array c c₃p > c₁p > c₃ˉp > c₁ˉp >\\c₂ˉp > c₁p > c₂p > c₂ˉp \endarray$};
\node[LPTnode] (P) at (1,1) {$P$} ; \node[LPTtable,anchor=west] at (P.east) {$ˉp > p$};
\node[LPTnode] (C) at (1,2) {$C$} ; \node[LPTtable,anchor=north] at (C.south) {$c₃ > c₁ > c₂$};
\draw (CP) to node[below,pos=.9] {$ˉw$} (W)to node[below,pos=.15] {$w$}(P)--(C) ;
\end{tikzpicture}
\end{center}
\begin{itemize}

\item The root of $ φ₀ $ is labelled with attribute $ W $, meaning that this is the most important attribute in my decision. Associated to it is the \emph{local preference} $ \bar w > w $, that indicates that I would rather go now, irrespective of the other attributes: given two alternatives $ o $ and $ o' $ such that $ o[W] = \bar w $ and $ o'[W] = w $, $ o ≻ o' $.

\item On the right branch, the next node is labelled with $ P $; it is connected to its parent with an edge labelled with $ w $, and has the local preference $ \bar p > p $: it indicates that if I go later ($w$) then the second most important attribute is the means of transport, and that I'd rather avoid flying; so whatever the cities $ c_i $ and $ c_j $, $ w\bar pc_i ≻ wpc_j $.

\item Still on the right branch, the leave is labelled with $ C $ and has associated local preference $c₃ > c₁ > c₂$, indicating that, still when going later, if I must choose between alternatives that have the same means of transport, I prefer the alternative that has my most preferred destination $ c₃ $, whereas my least preferred destination is $ c₂ $; so for instance $ wpc₃ ≻ wpc₁ ≻ wpc₂ $.

\item The only node below the root on the left branch is labelled with the pair of attributes $ CP $. It is connected to its parent with an edge labelled with $ \bar w $, and has an associated local preference that linearly orders the cartesian product of the domains of $ C $ and $ P $. It indicates that, if I go now ($ \bar w $), I'd rather fly to $ c₁ $ than drive to $ c₃ $, but drive to $ c₂ $ rather than fly to $ c₁ $: $ \bar wpc₁ ≻ \bar w\bar pc₃ $ and $ \bar w\bar pc₂ ≻ \bar wpc₁ $.

\end{itemize}

Now, because the root indicates that $ W $ is the most important attribute and that $ \bar w > w $, every alternative that has $ W=\bar w $ will be preferred to every alternative that has $ W=w $. So the most preferred alternative is $ \bar wc₃ p $, whereas the least preferred one is $ w\bar pc₃ $. And $ w\bar pc₂ $ is preferred to $ wpc₃ $, because both alternatives have equal value for the attribute at the root ($ W = \bar w $ for both), and the first node where they have differing values is the one labelled with $ P $, and $ \bar p > p $ at that node.
\end{example}

\begin{definition} \label{def:lp-tree}
An LP-tree $ φ $ over $ \cal X  $ is a rooted tree with labelled nodes and edges, and a set of local preference tables; specifically
\begin{itemize}
\item every node $ N $ is labelled with a set of attributes, denoted $ §Var§(N) $;

\item if $ N $ is not a leaf, it can have one child, or $ | \dom {§Var§(N)} | $ children;

\item in the latter case, the edges that connect $ N $ to its children are labelled with the instantiations in $ \dom {§Var§(N)} $;

\item if $ N $ has one child only, the edge that connects $ N $ to its child is not labelled: all instantiations in $ \dom {§Var§(N)} $ lead to the same subtree;

\item a local preference table $ §CPT§(N) $ is associated with $ N $, it specifies a linear order over  $ \dom {§Var§(N)} $;

\item moreover, every attribute must appear exactly once on every branch of $ φ $.
\end{itemize}

We denote by $ §Anc§(N) $ the set of \emph{ancestors} of $ N $: the attributes that appear in the nodes between the root and $ N $ (excluding those at $ N $), and by $ §Inst§(N) $ (resp. $ §NonInst§(N) $) the set of attributes that appear in the nodes above $ N $ that have more than one children (resp. only one child). Finally, we denote by $ §Desc§(N) $ the set of \emph{descendents} of $ N $: the attributes that appear below $ N $; thus $ (§Anc§(N),§Var§(N),§Desc§(N)) $ is a partition of $ \cal X  $.

Given an LP-tree $ φ $ and an alternative $ o ∈ \dom{\cal X } $, there is a unique way to traverse the tree, starting at the root, and along edges that are either not labelled, or labelled with instantiations that agree with $ o $, until a leaf is reached.

LP-tree $ φ $ defines a linear order $ ≻ _φ $ over $ \domX $ as follows: given two distinct alternatives $ o, o' $, it is possible to traverse the tree along edges that correspond to $ o $ and $ o' $ as long as $ o $ and $ o' $ agree, until a node $ N $ is reached which is labelled with some $ W $ such that $ o[W] ≠ o'[W] $: we say that $ N $ decides $ \{o,o'\} $, and $ o ≻_φ o' $ if and only if $ o[W] > o'[W] $ in the linear order in $ §CPT§(N) $. In the sequel, slightly abusing notations for ease of readability, we will often write φ where $ ≻_φ $ would be expected: for instance, we will denote by $ §rank§(φ,o) $ the rank of alternative $ o $ in the linear order $ ≻_φ $; and write $ §rloss§_p(φ,̆φ) $  the ranking loss of $ ≻_φ $ with respect to target linear order $ ̆≻ $ if $ ̆φ $ is an LP-tree that represents $ ̆≻ $.
\end{definition}

Definition~\ref{def:lp-tree} corresponds to the $ k $-LP-trees of \cite{FargierGimenezMengin:aaai18}. It is more general than the definition of LP-trees of \cite{BrauningHullermeier:pl12} because it allows for nodes with one child only, and is more general than the definition of \cite{Boothetal:ecai10} in that it allows for more than one attribute at every node. (However, \cite{Boothetal:ecai10} allow \emph{incomplete} LP-trees, where some attributes may be missing on some branch, leading to partial orders over $ \domX $.)  Every linear order can be represented with an LP-tree as defined above, possibly with a single node that contains all attributes at the root.

\cite{FargierGimenezMengin:aaai18} give a greedy, top-down unsupervised algorithm to learn such an LP-tree from a multiset of positive examples, and describe some experiments.

\section{Computing the rank expectation}
\label{sect:rank}


Many results in the sequel crucially rely on the possibility to decompose the rank of any alternative $ o $ in (the order represented by) any LP-tree. In this section, we slightly modify the decomposition given by \cite{LangMenginXia:artint18} in order to compute the rank expectation over all alternatives.  We illustrate this on an example first, and then give the general formula.

%
%

\begin{example}

Consider the LP-tree $ φ₀ $ of Example~\ref{exple:CP-th}, and suppose that we want to compute the rank of alternative $ wc₁p $: this amounts to counting the number of alternatives that are ``to its left'' in $ φ₀ $. $ wc₁p $ is less preferred than all alternative that have $ W=\bar w $, there are $ \card{\dom{CP}} = 3 × 2 = 6 $ of them. Alternatives that have $ W=w $ and $ P=\bar p $ are also preferred to $ wc₁p $, there are $ \card{\dom C} = 3 $ of them. $ wc₁p $ is also less preferred than $ wc₃p $. Finally $ §rank§(φ₀ , wc₁p) = 1 + (6 + 3 + 1)  = 11 $.

In general, the rank of any alternative can be decomposed as (one plus) a sum of contributions at every node: the root of $ φ₀ $ contributes 6 to the rank of $ wc₁p $, the left-most leaf labelled $ CP $ contributes $ 0 $ because it is not on the branch that corresponds to $ wc₁p $, the node labelled $ P $ contributes 3, and the node labelled $ C $ contributes 1.
\end{example}


More generally, given LP-tree $ φ $ and alternative $ o $, it can be shown that:
{\iftc\arraycolsep0pt\begin{align}
  & §rank§(φ,o) = 1 + \sum_{{N ∈ §nodes§(φ)}}  ⟦o[§Inst§(N)] = §inst§(N) ⟧ \label{eq:rank-decomp}
   \\&\nonumber\qquad\qquad× ( r(§Var§(N)>_N,o[§Var§(N)]) - 1) × \card{\dom{§Desc§(N)}} 
\end{align}
\else
\begin{equation}
   §rank§(φ,o) = 1 + \label{eq:rank-decomp}
   \sum_{\mathclap{N ∈ §nodes§(φ)}}  ⟦o[§Inst§(N)] = §inst§(N) ⟧ \nonumber
   × \big( r(§Var§(N), >_N,o[§Var§(N)]) - 1\big) 
   × \card{\dom{§Desc§(N)}} 
\end{equation}
\fi}
where :
\begin{itemize}

\item  $ §nodes§(φ) $ denotes the set of nodes of LP-tree $ φ $;

\item $ ⟦o[§Inst§(N)] = §inst§(N) ⟧ $ is an indicator function, that equals 1 when the condition $ o[§Inst§(N)] = §inst§(N)$ is true; that is, when $ N $ is on the branch of $ φ $ that corresponds to $ o $; and equals 0 otherwise;

\item $ >_N $ is the linear order over $ \dom{§Var§(N)} $ specified in $ §CPT§(N) $;

\item $ r(§Var§(N),>_N,o[§Var§(N)]) $ denotes the rank in $ §Var§(N) $ with respect to $ >_N $ of the instantiation given by $ o  $ to $ §Var§(N) $; so that $ r(§Var§(N)>_N,o[§Var§(N)]) - 1 $ is the number of subtrees rooted at children of $ N $ that are less preferred than $ o $ at $ N $;

\item $ §Desc§(N) = \cal X  - (§Anc§(N)∪§Var§(N)) $ is the set of attributes that appear below $ N $ in that branch, so that $ \card{\dom{§Desc§(N)}} $ is the number of instantiations that are ``contained" in every subtree of $ φ $ rooted at a child of $ N $.

\end{itemize}

The difference with the decomposition of \cite{LangMenginXia:artint18} is that we sum over all nodes of $ φ $, irrespective of the alternative. This is useful to express the expectation of the rank of φ with respect to some probability distribution $ p $ over $ \domX $. For set of attributes $ V ⊆ \cal X  $, linear order $ > $ over $ \dom V $, and $ v ∈ \dom V $, let $ r(V,>,v) $ be the rank of $ v $ in $ \dom V $ with respect to $ > $. Then the expectation of this rank, w.r.t. $ p $, is
$$
  E_p[r(V,>,⋅)] = \sum_{v ∈ \dom V} p(v)r(V,>,v)
$$
where $ p(v) $ denotes the probability of drawing an alternative $ o $ such that $ o[V] = v $. 


\begin{propositionrep}
\label{prop:rank-decomp}
For LP-tree φ and probability distribution $ p $:
\begin{align}
&   E_p[§rank§(φ,
⋅)] = 1 +
\\&\nonumber\qquad    \sum_{{N ∈ §nodes§(φ)}} \card{\dom{§Desc§(N)}}× p(§inst§(N)) 
\\&\nonumber\qquad\qquad\qquad   × E_{p|§inst§(N)}[r(§Var§(N),>_N,⋅) - 1] 
\end{align}
where $ p|§inst§(N) $ denotes the probability distribution marginalized to $ §inst§(N) $, that is: $ p|§inst§(N)(v) = p(v |§inst§(N)) $.
\end{propositionrep}

\begin{proofsketch}
By definition, $ E_p[§rank§(φ,⋅)] = \sum_{o ∈ \domX } \big(p(o) × §rank§(φ,o)\big) $. Equation~\ref{eq:rank-decomp} shows that $ §rank§(φ,o) $ can be decomposed as a sum of contributions over all nodes of $ φ $. It is not difficult to see that it is possible to invert the sum over $ \domX $ and the sum over $ §nodes§(φ) $, which yields the result.
\end{proofsketch}

\begin{proof}
\compactmath
\begin{align*}
  E_p[§rank§(φ,⋅)] & = \sum_{o ∈ \domX } \big(p(o) × §rank§(φ,o)\big)
\\&= 1 + \sum_{o ∈ \domX } \bigg(p(o) × \sum_{\mathclap{N ∈ §nodes§(φ)}}
    ⟦o[§Inst§(N)] = §inst§(N) ⟧ × ( r(§Var§(N),>_N,o[§Var§(N)]) - 1) × \card{\dom{§Desc§(N)}}\bigg)
\\&= 1 + \sum_{o ∈ \domX } \qquad \sum_{\mathclap{N ∈ §nodes§(φ)}} \cramped{\bigg(p(o) × 
    ⟦o[§Inst§(N)] = §inst§(N) ⟧ × ( r(§Var§(N),>_N,o[§Var§(N)]) - 1) × \card{\dom{§Desc§(N)}}\bigg)}
\\&= 1 + \sum_{N ∈ §nodes§(φ)} \bigg(\card{\dom{§Desc§(N)}} \qquad
	\sum_{\mathclap{\substack{o ∈ \domX \\o[§Inst§(N)] = §inst§(N)}}} p(o) × (r(§Var§(N),>_N,o[§Var§(N)]) - 1)\bigg)
\\&= 1 + \sum_{N ∈ §nodes§(φ)} \bigg( \card{\dom{§Desc§(N)}} ×
    \sum_{v ∈ \dom{§Var§(N)}} \big( (r(§Var§(N),>_N,v) - 1)
	\sum_{\substack{o ∈ \domX \\o[§Inst§(N)] = §inst§(N)\\o[§Var§(N)=v]}} p(o) \big) \bigg)
\\&= 1 + \sum_{N ∈ §nodes§(φ)} \bigg( \card{\dom{§Desc§(N)}} ×
    \sum_{v ∈ \dom{§Var§(N)}} p(v∧§inst§(N))(r(§Var§(N),>_N,v) - 1)\bigg)
\\&= 1 + \sum_{N ∈ §nodes§(φ)} \bigg( \card{\dom{§Desc§(N)}} × p(§inst§(N))
    \sum_{v ∈ \dom{§Var§(N)}} p(v|§inst§(N))(r(§Var§(N),>_N,v) - 1)\bigg)
\end{align*}
\end{proof}

\paragraph{Locally optimal LP-trees}

Proposition~\ref{prop:rank-decomp} above suggests that, given some LP-tree $ φ $, and a probability distribution $ p $, one can reduce the rank expectation of $ φ $ by re-ordering the values at every node $ N $, so as to minimize $ E_{p|§inst§(N)}[r(§Var§(N),>_N,⋅)] $: essentially, one only has to order the instantiations  of $ §Var§(N) $ in order of non-increasing probability. Thus we have here a desirable property shared by many graphical models of preferences, namely that learning can often be decomposed in two parts: 1) learning the structure, often a hard problem; and 2) learning the \emph{local} preferences, often an easy task.

\begin{definition}
Given a probability distribution $ p $ over $ \domX $, we say that LP-tree $ φ $ is \emph{$p$-locally optimal} if at every node $ N $ of $ φ $, the linear order $ >_N $ orders the instantiations of $ §Var§(N) $ in order of non-increasing probability.
\end{definition}

In the sequel, given set of attributes $ V ⊆ \cal X  $, and probability distribution $ p $ over $ \domX $, we denote by $ E^∗_p(V) $ the minimum expectation that is possible for the ranks of the instantiations of $ V $ given $ p $, that is, when ordering these values in order of non-increasing probability $ p $:
$$
  E^∗_p(V) = \argmin_{>}E_p[r(V,>,⋅)]
$$


In the next sections, we study, for different classes of LP-trees, how difficult it is to compute an LP-tree that has minimum mean rank with respect to sample $ \cal S $, and how many alternatives in $ \cal S $ guarantee that this $ \cal S $-optimal LP-tree is close to the target one.

\section{Linear lexicographic preference trees with univariate nodes}

We consider in this section the class of \emph{linear} LP-trees, that is, LP trees that have a single branch; and also such that every node is labelled with a single attribute. Let us denote by $ §LPT§¹_{§lin§} $ this class. It is in fact the class of the usual lexicographic preference relations \cite{Fishburn:managsc74}. 

\subsection{Greedy algorithm for the optimal LP-tree}

Note first that at every node $ N $ of some $ φ ∈ §LPT§¹_{§lin§} $, $ §Inst§(N) $ is empty, so $ §inst§(N) $ is the empty instantiation, and $ p|§inst§(N) = p $. Thus :
\begin{align*}
   &E_p[§rank§(φ,⋅)] = 1 +
   \\&\qquad\sum_{{N ∈ §nodes§(φ)}}
     \card{\dom{§Desc§(N)}} × E_p[r(§Var§(N),>_N,⋅) - 1] 
\end{align*}
 
 We define, for every attribute $ X ∈ \cal X  $ and probability distribution $ p $ over $ \domX$:
 $$
   §Score§(p,X) = E^∗_p(X) \mathrel / (\card{\dom X} - 1) 
 $$


\begin{proposition} \label{prop:score-lin-univ}
Let $ p $ be a probability distribution over $ \domX$, the LP-tree that has minimal expected rank in $ §LPT§¹_{§lin§} $ has the attributes in order of non-decreasing values for $ §Score§(p,⋅) $ along its single branch from the root down to its single leaf, and is locally optimal.
\end{proposition}

\begin{proof}
Let $ φ $ be a locally optimal tree of $ §LPT§¹_{§lin§} $. Suppose that there are two nodes of φ, $ N $ and $ N' $, labelled with $ X $ and $ Y $ respectively, such that $ N' $ is the child of $ N $ but $ §Score§(p,X)> §Score§(p,Y) $. Let $ φ' $ be identical to $ φ $ except that $ X $ and $ Y $ have been inverted: $ X $ is at $ N' $ in $ φ' $, and $ Y $ at $ N $. Then $ §Desc§(N) = §Desc§(N') ∪ §Var§(N') $, where $ §Var§(N') = \{Y\} $ in $ φ $, and$ §Var§(N') = \{X\} $ in $ φ' $. Thus:
$$ \begin{array}{l}
  E_p[§rank§(φ,⋅)]  - E_p[§rank§(φ',⋅)] = \\ \qquad \card{\dom{§Desc§(N')}} ×\big(
    E^∗_p(X) (\card{\dom Y} - 1) - E^∗_p(Y) (\card{\dom X} - 1)
  \big)
\end{array}$$
On the other hand:
\begin{align*}
  & §Score§(p,X) > §Score§(p,Y) \\
  & ⇔ \frac{E^∗_p(X)}{\card{\dom X} - 1} - \frac{E^∗_p(Y)}{\card{\dom Y} - 1} >0
\\& ⇔ E^∗_p(X)×(\card{\dom Y} - 1) - E^∗_p(Y)×(\card{\dom X} - 1) > 0
\\& ⇔ E_p[§rank§(φ,⋅)]  - E_p[§rank§(φ',⋅)] > 0
\end{align*}
This shows that, by applying a kind of \emph{bubble sort} to the nodes of $ φ $ so as to order them in order of non-decreasing score, we obtain a new tree that has better expected rank.
\end{proof}

As a consequence:

\begin{corollary} \label{ppty:lin-univ-score-target}
Let $ ̆φ $ be some target LP-tree in $ §LPT§¹_{§lin§} $ and $ p $ some probability distribution over $ \domX $ non-decreasing w.r.t. $ §rank§(̆φ,⋅) $, then the attributes appear in $ ̆φ  $ in order of non-decreasing values for $ §Score§(p,⋅) $ from the root down to the leaf.
\end{corollary}

\begin{corollary} \label{ppty:incr-score-target}
Let $ \cal S  ⊆ \domX $, recall that $  \probS $ denotes the probability distribution over $ \domX$ such that $ \probS(o) $ is the frequency of $ o $ in $ \cal S $.
Let $ φ^∗ $ be the LP-tree in $ §LPT§¹_{§lin§} $ that has minimal mean rank w.r.t. $ \cal S  $, then then the attributes appear in $ φ^∗ $ in order of non-decreasing values for $ §Score§(\probS,⋅) $ from the root down to the leaf.
\end{corollary}


\subsection{Time and sample complexity}

According to Corollary~\ref{ppty:lin-univ-score-target} above, in order to compute $ φ ^∗ $, we can,  for every $ X ∈ \cal X  $, compute the number of occurrences in $ \cal S  $ of every $ x ∈ \dom X $, and compute $>_X^{\probS} $ and $ §Score§(\probS,X) $, and then order the attributes w.r.t. $ §Score§(\probS,⋅) $. Thus:

\begin{proposition} \label{ppty:lin-univ-score-opt}
Given sample $ \cal S  ⊆ \domX $, computing $ φ ^∗ $ in $ §LPT§¹_{§lin§} $ can be done in time in $ O(n\log n\card{\cal S }d\log d) $, where $ d $ is an upper bound on the size of domains of each attribute $ X ∈ \cal X  $.
\end{proposition}

\begin{proposition}
$ S(§LPT§^1_{§lin§},n,δ,ε) = O(\frac1{ε²}(\ln\frac1{δ}+\ln d + \ln(n+1)))d⁴) $.
\end{proposition}

\begin{proof}
 This is simple consequence of the more general Proposition~\ref{prop:cas-general:sample-complex} given in section~\ref{sect:cas-general} for the sample complexity of learning LP-trees with bounded number of leaves and bounded number of attributes at every node: here both bounds equal 1.
\end{proof}

\section{Linear lexicographic preference trees with multivariate nodes}


We consider now the more general settings of the class $ §LPT§_{§lin§}^k $, for some fixed $ k ≤ n $: linear LP-trees (with a single leaf), where each node can be labelled with a non-empty set of no more than $ k $ attributes.

The score defined in the case of univariate nodes in the previous section can easily be extended to multivariate nodes: for $ V ⊆ \cal X  $ and probability distribution $ p $ over $ \domX$:
$$
  §Score§(p,V) = E^∗_p[V] \mathrel / (\card{\dom V} - 1)
$$
Proposition~\ref{prop:score-lin-univ} can be generalised in this settings as follows:

\begin{proposition} \label{prop:score-lin-multiv}
Let $ p $ be a probability distribution over $ \domX$, let $ \cal P $ be a partition of $ \cal X $. The LP-tree that has minimal expected rank, within the class of LP-trees with a single leaf, and where each node is labelled with an element of $ \cal P $, is locally optimal and has the set of attributes that label its nodes in order of non-increasing values for $ §Score§(p,⋅) $ along its single branch from the root down to its single leaf.
\end{proposition}

\begin{proof}
Given partition $ \cal P $ of $ \cal X $, consider each part $ P ∈ \cal P $ as a new attribute with domain $ \dom P $, Prop.~\ref{prop:score-lin-univ} gives the result.
\end{proof}

Based on this, algorithm~\ref{algo:opt-tree-lin-multiv} computes the tree in $ §LPT§_{§lin§}^k $ that minimises the empirical mean rank with respect to some sample $ \cal S $: it enumerates the \emph{$k$-partitions} of $ \cal X $, i.e. the partitions such that no part has more than $ k $ attributes; and, for every partition, orders the parts in order of non-increasing score and compute the empirical mean rank of the tree thus obtained.

\begin{algorithm}
\caption{\label{algo:opt-tree-lin-multiv}Optimal linear multivariate LP-tree}
\begin{steps}
\noitem Input : sample $ \cal S  $;
\item $ r^∗ ← +∞ $;
\item for every $k$-partition $ \cal P  $ of $ \cal X  $ do: \label{step:order-parts-bis}
	\begin{steps}
	\item \label{step:order-parts} order the parts in $ \cal P $ in order of non-decreasing $ §Score§(\probS,⋅) $;
	\item $ φ ← $ the optimal linear univariate LP-tree over set of attributes $ \cal P  $;
	\item if $ \erank(φ,\cal S ) < r^∗ $: $ r^∗ ← \erank(φ,\cal S ) $ and $ φ^∗ ← φ $;
	\end{steps}
\item return $ φ^∗ $.
\end{steps}
\end{algorithm}

\begin{proposition}
For fixed $ k $, the time complexity of Algorithm~\ref{algo:opt-tree-lin-multiv} is in $ Ω(k^{n}n(\log n +\card{\cal S} d^k\log d)) $ and in $ O(k^{n+1}n(\log n +\card{\cal S} d^k\log d) $.
\end{proposition}

\begin{proof}
The number of $k$-partitions of $ \cal X $ is also the number of partitions of $ \cal X $ into at most $ k $ non empty parts (in the same way as the number of partitions of integer $  n $ into positive integers no greater than $ k $ is equal to the number of partitions of $ n $ into $ k $ non-negative integers). Let $ S(n,k) $  denote the number of partitions of a set of $ n $ elements into exactly $ k $ non-empty parts, it is known as the \emph{Stirling number of the second kind}, an asymptotic approximation is $ S(n,k) ∼ k^n / k! $ \citep[e.g. ][]{Knuth:-the-art-I}. Therefore, for fixed $ k $, the number of $ k $-partitions of $ \cal X $ is asymptotically equivalent to $ \sum_{i=1}^k i^n / i! $, which is $ ≥ k^n / k! $.
We also have that $ i^n / i! < j^n / j! $ for $ i < j $, thus $ \sum_{i=1}^k i^n / i! < k × k^n / k!  $. So the number of $k$-partitions is in $ Θ(k^n) $.

For a given partition, at step~\ref{step:order-parts-bis}\ref{step:order-parts}, one has to compute the scores of the parts, and order them. There are no more than $ n $ parts; each part has no more than $ d^k $ possible instantiations, so computing and ordering the numbers of occurrences of for all instantiations takes worse time in $ Θ(n\card{\cal S} d^k\log (d^k)) $. Ordering the parts then takes time in $ Θ(n  \log n) $.
\end{proof}

\section{Lexicographic preference trees with bounded number of leaves}
\label{sect:cas-general}

We now turn to the class of LP-trees with a bounded number of attributes at each node, and a bounded number of leaves. Specifically, $ §LPT§^k_l $ denotes the set of LP-trees that have no more than $ k $ attributes at each node, and no more than $ l $ leaves. (In particular, when $ l = 1 $, $ §LPT§^k_1 = §LPT§^k_{§lin§} $

\begin{proposition}\label{prop:cas-general:sample-complex}
$ S(§LPT§^k_l,n,δ,ε) = O(\frac1{ε²}(\ln\frac1{δ}+k(\ln d + \ln(n+1)))l²d^{4k}) $.
\end{proposition}


\begin{proof}
Recall that $ ̆φ $ is an LP-tree, supposed to be in $ §LPT§_l^k $, that represents the target, unknown preference relation $ ̆≻ $; $ p $ is a probability distribution that is non-increasing w.r.t. $ §rank§(̆φ,⋅) $; and
$ \phi^* $ denotes the LP-tree in $ §LPT§_l^k $ that has minimal normalised empirical mean rank with respect to some sample $ \cal S  $. Then:
\begin{align*}
 & §rloss§(φ^∗,̆φ)  = \frac1{\card{\domX}}\big(E_p[§rank§(φ^∗,⋅)] - E_p[§rank§(̆φ,⋅)] \big)
\\&\qquad≤\frac1{\card{\domX}} \big(\abs{E_p[§rank§(φ^∗,⋅)] - E_{\probS}[§rank§(φ^∗,⋅)]}
	\\&\qquad\qquad+ E_{\probS}[§rank§(φ^∗,⋅)] - E_{\probS}[§rank§(̆φ,⋅)] 
	\\&\qquad\qquad+ \abs{E_{\probS}[§rank§(̆φ,⋅)] - E_p[§rank§(̆φ,⋅)]}  \big)
\\&\qquad≤ \frac2{\card{\domX}} \max_{φ∈§LPT§_l^k} \abs{E_p[§rank§(φ,⋅)] - E_{\probS}[§rank§(φ,⋅)]}
\end{align*}
because, by definition of $ φ^∗ $:
$$ E_{\probS}[§rank§(φ^∗,⋅)] - E_{\probS}[§rank§(̆φ,⋅)] ≤ 0 .$$

Now, for $ φ ∈ §LPT§_l^k $ :
\begin{align*}
&\abs{E_p[§rank§(φ,⋅)] - E_{\probS}[§rank§(φ,⋅)]}
\\& \qquad≤
  \qquad\sum_{\mathclap{N ∈ §nodes§(φ)}}
     \bigg(\card{\dom{§Desc§(N)}} × p(§inst§(N))
     \\&\qquad\qquad× \vert E_{p|§inst§(N)}[r(§Var§(N),>_N,⋅)] 
     \\&\qquad\qquad\qquad- E_{\probS|§inst§(N)}[r(§Var§(N),>_N,⋅)]\vert\bigg)
\\&≤ \sum_{{N ∈ §nodes§(φ)}} \bigg(
     \card{\dom{§Desc§(N)}} × \sum_{\mathclap{v∈\dom{§Var§(N)}}} \big(r(v)
     \\&\qquad \qquad \qquad \qquad× \absv{p(v\!∧\!§inst§(N)) \!-\! \probS(v\!∧\!§inst§(N))} \big)\bigg)
\\&≤ M×\frac{d^k(d^k+1)}2 × \sum_{\mathclap{N ∈ §nodes§(φ)}} \card{\dom{§Desc§(N)}} 
\end{align*}
where $ d $ is a bound on the domain size of the attributes, and $ M $ is an upper bound on $ \absv{p(v∧§inst§(N)) - \probS(v∧§inst§(N))} $ for every $ N ∈ §nodes§(φ) $, every $ v ∈ \dom{§Var§(N)} $.

Now consider a single branch $ ψ $ of $ φ $, let $ V₁, V₂, …, V_t $ be the set of attributes that label its $ t $ nodes from the root down to its leaf (at unknown depth $ t $), then:
{\compactmath
\begin{align*}
  &\qquad\sum_{\mathclap{N ∈ §nodes§(ψ)}} \card{\dom{§Desc§(N)}}
  \\&= \big( \card{\dom V₂× \dom V₃× … \dom V_t } 
  	+ … + \card{\dom V_{t-1}×\dom V_t } +  \card{\dom V_t } + 0 \big)
\\&= \card{\domX} × \big( \frac1{\card{\dom V₁}} + \frac1{\card{\dom V₁ ×\dom V₂}} 
	+…+  \frac1 {\card{\dom V₁ ×\dom V₂×…×\dom V_{t-1}}} \big)
\\&≤ \card{\domX}
\end{align*}
because for every node, $ \card{\dom{§Var§(N)}}  ≥ 2$. Therefore, since we  consider LP-trees with a bounded number of leaves (and branches), for every $ φ ∈ §LPT§_l^k $:
$$
\abs{E_p[§rank§(φ,⋅)] - E_{\probS}[§rank§(φ,⋅)]} ≤ M×\frac{d^k(d^k+1)}2 × l× \card{\domX}
$$
and
$
  §rloss§_p(φ^∗,̆φ)  
≤  l×d^k(d^k+1) × \max\limits_{\substack{V⊆\cal X ,v∈\dom V}} \abs{p(v) - \probS(v)}
$.
}

Thus, if $ §rloss§_p(φ^∗,̆φ) ≥ ε $, there must be some $ V ⊆ \cal X  $ and $ v ∈ \dom V $ such that
$
  \abs{p(v) - \probS(v)} ≥ ε / (l×d^k(d^k+1))
$,
which implies that:
\begin{align*}
  &Pr(§rloss§(φ^∗,̆φ) ≥ ε) \\& ≤ Pr\big(⋃_{\substack{V ⊆ \cal X \\v\in \dom V}}\abs{p(v) - \probS(v)} ≥ ε / (l×d^k(d^k+1))\big)
\\& ≤ \sum_{\substack{V ⊆ \cal X \\v ∈ \dom V}} Pr\big(\abs{p(v) - \probS(v)} ≥ ε / (l×d^k(d^k+1))\big)
\end{align*}

For every $ V ⊆ \cal X  $ and every $ v ∈ \dom V $, $ \probS(v) $ is an estimate, from sample $ \cal S  
$, of the ground probability $ p(v) $ of drawing an alternative $ o $ such that $ o[V]= v $. Hoeffding's inequality states that for every $ α > 0 $:
$$
	Pr(\abs{p(v) - \probS(v)} ≥ α) ≤ e^{-2\card{\cal S }α²}
$$

For every $i \in \{1,…,k\}$, there are $n \choose i$ ways of choosing a subset $ V $ of $ \cal X  $ of cardinality $ i $, then $ \card{\dom V} ≤ d^i $; therefore:

\begin{align*}
  & Pr(§rloss§(φ^∗,̆φ) ≥ ε)
  \\& ≤ \bigg(\sum_{i=1}^k {n \choose i} d^i\bigg) \exp(-2\card{\cal S }(ε / (l×d^k(d^k+1)))²)
\\& ≤  d^k(1+n)^k  \exp(-2\card{\cal S }(ε / (l×d^k(d^k+1)))²)
\end{align*}

Therefore, in order to have $ Pr(§rloss§(φ^∗,̆φ) ≤ ε) ≥ 1 - δ  $, it is sufficient to have $  1-d^k(1+n)^k  \exp(-2\card{\cal S }(ε / (l×d^k(d^k+1)))²) ≥ 1 - δ$, which is equivalent to:

$
\card{\cal S } ≥ \big(k(\ln d + \ln(n+1))) + \ln\displaystyle\frac1{δ}\big)\frac{(ld^k(d^k+1))²}{2ε²}
$
\end{proof}

\section{Conclusion}

The bound on the sample complexity given in the last section shows several interesting properties of the problem of unsupervised learning LP-trees from sales history. First, it is logarithmic in the number of attributes. Also the factor $ d^{4k} $ may, in general, be overly pessimistic, as it assumes that all attributes have the same domain size. A finer analysis would show that what counts is the largest domain size of any combination of $ k $ attributes. Note that $ k $ should be kept small: \cite{Brauningetal:ejor17,FargierGimenezMengin:aaai18} report promising results on real datasets with $ k < 4 $; larger values of $ k $ would improve the expressiveness, but also greatly increase the model size and may lead to overfitting.

Concerning the time complexity, we conjecture that the problem of computing the LP-tree that has minimal empirical risk may be NP-hard in general, but that is an important question for future work. Another avenue for future investigation is a finer analysis of the stochastic process that leads to the probability distribution $ p $.

\section*{Acknowledgements}
We thank the reviewers for their interesting and helpful comments.
The authors gratefully acknowledge the support of the Artificial and Natural Intelligence Toulouse Institute -- ANITI. ANITI is funded by the French ”Investing for the Future -- PIA3” program under grant agreement ANR-19-PI3A-0004.


\begin{thebibliography}{31}
\providecommand{\natexlab}[1]{#1}
\providecommand{\url}[1]{\texttt{#1}}
\expandafter\ifx\csname urlstyle\endcsname\relax
  \providecommand{\doi}[1]{doi: #1}\else
  \providecommand{\doi}{doi: \begingroup \urlstyle{rm}\Url}\fi

\bibitem[Booth et~al.(2010)Booth, Chevaleyre, Lang, Mengin, and
  Sombattheera]{Boothetal:ecai10}
Richard Booth, Yann Chevaleyre, J{\'e}r{\^o}me Lang, J{\'e}r{\^o}me Mengin, and
  Chattrakul Sombattheera.
\newblock Learning conditionally lexicographic preference relations.
\newblock In \emph{Proceedings of {ECAI}'10}, pages 269--274, 2010.

\bibitem[Boutilier et~al.(2004{\natexlab{a}})Boutilier, Brafman, Domshlak,
  Hoos, and Poole]{Boutilier04cp-nets}
Craig Boutilier, Ronen~I. Brafman, Carmel Domshlak, Holger~H. Hoos, and David
  Poole.
\newblock {CP}-nets: A tool for representing and reasoning with conditional
  ceteris paribus preference statements.
\newblock \emph{Journal of Artificial Intelligence Research}, 21:\penalty0
  135--191, 2004{\natexlab{a}}.

\bibitem[Boutilier et~al.(2004{\natexlab{b}})Boutilier, Brafman, Domshlak,
  Hoos, and Poole]{Boutilieretal:compint04}
Craig Boutilier, Ronen~I. Brafman, Carmel Domshlak, Holger~H. Hoos, and David
  Poole.
\newblock Preference-based constrained optimization with cp-nets.
\newblock \emph{Computational Intelligence}, 20\penalty0 (2):\penalty0
  137--157, 2004{\natexlab{b}}.

\bibitem[Br{\"a}uning and H{\"u}llermeyer(2012)]{BrauningHullermeier:pl12}
Michael Br{\"a}uning and Eyke H{\"u}llermeyer.
\newblock Learning conditional lexicographic preference trees.
\newblock In Johannes F{\"u}rnkranz and Eyke H{\"u}llermeyer, editors,
  \emph{Proceedings of {ECAI}'12 Workshop}, 2012.

\bibitem[Br{\"{a}}uning et~al.(2017)Br{\"{a}}uning, H{\"{u}}llermeier, Keller,
  and Glaum]{Brauningetal:ejor17}
Michael Br{\"{a}}uning, Eyke H{\"{u}}llermeier, Tobias Keller, and Martin
  Glaum.
\newblock Lexicographic preferences for predictive modeling of human decision
  making: {A} new machine learning method with an application in accounting.
\newblock \emph{European Journal of Operational Research}, 258\penalty0
  (1):\penalty0 295--306, 2017.

\bibitem[Braziunas(2005)]{Braziunas05localutility}
Darius Braziunas.
\newblock Local utility elicitation in {GAI} models.
\newblock In \emph{Proceedings of UAI'05}, pages 42--49, 2005.

\bibitem[Brewka et~al.(2006)]{Wilson:ecai06}
G~Brewka et~al.
\newblock An efficient upper approximation for conditional preference.
\newblock In \emph{Proceedings of {ECAI}'06}, volume 141, page 472, 2006.

\bibitem[Dombi et~al.(2007)Dombi, Imreh, and Vincze]{Dombietal:ejor07}
J{\'o}zsef Dombi, Csan{\'a}d Imreh, and N{\'a}ndor Vincze.
\newblock Learning lexicographic orders.
\newblock \emph{European Journal of Operational Research}, 183:\penalty0
  748--756, 2007.

\bibitem[Fargier and Mengin(2018)]{FargierMengin:aamas21}
H{\'e}l{\`e}ne Fargier and J{\'e}r{\^o}me Mengin.
\newblock A knowledge compilation map for conditional preference
  statements-based languages.
\newblock In \emph{Proceedings {AAMAS} 2021}, pages 492--500, 2018.

\bibitem[Fargier et~al.(2018)Fargier, Gimenez, and
  Mengin]{FargierGimenezMengin:aaai18}
H{\'e}l{\`e}ne Fargier, Pierre~Francois Gimenez, and J{\'e}r{\^o}me Mengin.
\newblock Learning lexicographic preference trees from positive examples.
\newblock In \emph{Proceedings {AAAI} 2018}, pages 2959--2966. {ACM}, 2018.

\bibitem[Fishburn(1974)]{Fishburn:managsc74}
Peter~C. Fishburn.
\newblock Lexicographic orders, utilities and decision rules: A survey.
\newblock \emph{Management Science}, 20\penalty0 (11):\penalty0 1442--1471,
  1974.

\bibitem[Fraser(1993)]{fraser:ieeeconf93}
Niall~M Fraser.
\newblock Applications of preference trees.
\newblock In \emph{Proceedings of {SMC}'93}, pages 132--136, 1993.

\bibitem[Fraser(1994)]{fraser:theo-dec94}
Niall~M. Fraser.
\newblock Ordinal preference representations.
\newblock \emph{Theory and Decision}, 36\penalty0 (1):\penalty0 45--67, 1994.

\bibitem[Freund et~al.(2003)Freund, Iyer, Schapire, and
  Singer]{Freundetal:jmlr03}
Yoav Freund, Raj~D. Iyer, Robert~E. Schapire, and Yoram Singer.
\newblock An efficient boosting algorithm for combining preferences.
\newblock \emph{Journal of Machine Learning Research}, 4:\penalty0 933--969,
  2003.

\bibitem[F{\"u}rnkranz and H{\"u}llermeier(2011)]{FurnkranzHullermeier:book11}
Johannes F{\"u}rnkranz and Heike H{\"u}llermeier, editors.
\newblock \emph{Preference learning}.
\newblock Springer, 2011.

\bibitem[Gigerenzer and Goldstein(1996)]{GigerenzerG:psych-reviewG96}
Gerd Gigerenzer and Daniel~G. Goldstein.
\newblock Reasoning the fast and frugal way: Models of bounded rationality.
\newblock \emph{Psychological Review}, 103\penalty0 (4):\penalty0 650--669,
  1996.

\bibitem[Gonzales and Perny(2004)]{GonzalesPerny04}
Christophe Gonzales and Patrice Perny.
\newblock {GAI} networks for utility elicitation.
\newblock In \emph{Proceedings of KR'04}, pages 224--234, 2004.

\bibitem[Huffman and Kahn(1998)]{huffman1998variety}
Cynthia Huffman and Barbara~E Kahn.
\newblock Variety for sale: Mass customization or mass confusion?
\newblock \emph{Journal of retailing}, 74\penalty0 (4):\penalty0 491--513,
  1998.

\bibitem[Joachims(2002)]{Joachims:kdd02}
Thorsten Joachims.
\newblock Optimizing search engines using clickthrough data.
\newblock In \emph{Proceedings of {SIGKDD}'02}, pages 133--142, 2002.

\bibitem[Knuth(1997)]{Knuth:-the-art-I}
Donald~E. Knuth.
\newblock \emph{The Art of Computer Programming, Volume {I}: Fundamental
  Algorithms}.
\newblock Addison-Wesley, 3rd edition edition, 1997.
\newblock ISBN 0201896834.

\bibitem[Koriche and Zanuttini(2009)]{KoricheZ09}
Fr\'ed\'eric Koriche and Bruno Zanuttini.
\newblock Learning conditional preference networks with queries.
\newblock In \emph{Proceedings of IJCAI'09}, pages 1930--1935, 2009.

\bibitem[Lang et~al.(2018)Lang, Mengin, and Xia]{LangMenginXia:artint18}
Jérome Lang, Jérome Mengin, and Lirong Xia.
\newblock Voting on multi-issue domains with conditionally lexicographic
  preferences.
\newblock \emph{Artificial Intelligence}, 265:\penalty0 18–44, 2018.

\bibitem[Liu and Truszczynski(2015)]{LiuTruszczynski:aaai15}
Xudong Liu and Miroslaw Truszczynski.
\newblock Learning partial lexicographic preference trees over combinatorial
  domains.
\newblock In \emph{Proceedings of {AAAI}'15}, volume~15, pages 1539--1545,
  2015.

\bibitem[Neapolitan(2003)]{Neapolitan:book03}
Richard~E. Neapolitan.
\newblock \emph{Learning bayesian networks}.
\newblock Prentice Hall, 2003.

\bibitem[Schiex et~al.(1995)Schiex, Fargier, and Verfaillie]{ScFaVe1995.1}
Thomas Schiex, Hélène Fargier, and Gérard Verfaillie.
\newblock {Valued constraint satisfaction problems: Hard and easy problems}.
\newblock In \emph{Proceedings of IJCAI'95}, pages 631--637, 1995.

\bibitem[Schmitt and Martignon(2006)]{SchmittMartignon:jmlr06}
Michael Schmitt and Laura Martignon.
\newblock On the complexity of learning lexicographic strategies.
\newblock \emph{Journal of Machine Learning Research}, 7:\penalty0 55--83,
  2006.

\bibitem[Viappiani et~al.(2006)Viappiani, Faltings, and Pu]{ViappianiAlJAIR06}
Paolo Viappiani, Boi Faltings, and Pearl Pu.
\newblock Preference-based search using example-critiquing with suggestions.
\newblock \emph{Journal of Artificial Intelligence Research}, 27:\penalty0
  465--503, 2006.

\bibitem[Wallace and Wilson(2009)]{WallaceW:ann-OR09}
Richard~J. Wallace and Nic Wilson.
\newblock Conditional lexicographic orders in constraint satisfaction problems.
\newblock \emph{Annals of Operations Research}, 171\penalty0 (1):\penalty0
  3--25, 2009.

\bibitem[Wilson(2004)]{Wilson:ecai04}
Nic Wilson.
\newblock Consistency and constrained optimisation for conditional preferences.
\newblock In \emph{Proceedings {ECAI} 2004}, pages 888--892, 2004.

\bibitem[Wilson(2011)]{Wilson:aij11}
Nic Wilson.
\newblock Computational techniques for a simple theory of conditional
  preferences.
\newblock \emph{Artificial Intelligence}, 175:\penalty0 1053--1091, 2011.

\bibitem[Yaman et~al.(2008)Yaman, Walsh, Littman, and
  Desjardins]{Yamanetal:icml08}
Fusun Yaman, Thomas~J Walsh, Michael~L Littman, and Marie Desjardins.
\newblock Democratic approximation of lexicographic preference models.
\newblock In \emph{Proceedings of {ICML}'08}, pages 1200--1207, 2008.

\end{thebibliography}

\end{document}